\documentclass[conference]{IEEEtran}
\IEEEoverridecommandlockouts
\usepackage{cite}
\usepackage{amsmath,amssymb,amsfonts}
\usepackage{algorithmic}
\usepackage{graphicx}
\usepackage{textcomp}
\usepackage{xcolor}
\usepackage[mathscr]{eucal}
\usepackage{amsbsy,mathrsfs}

\usepackage{microtype}
\usepackage{amssymb}
\usepackage{amsmath}
\usepackage{booktabs}       
\usepackage{amsfonts}       
\usepackage{nicefrac}       
\usepackage{microtype}      
\usepackage[utf8]{inputenc}
\usepackage{multirow}
\usepackage{float}
\usepackage[thinc]{esdiff}

\usepackage{amsmath,amssymb,bm}
\usepackage{algorithm}
\usepackage{algorithmic}
\usepackage{arydshln}
\usepackage{multicol}
\usepackage{verbatim}
\usepackage{tikz}
\usepackage{longtable}
\usepackage{amsthm}
\usepackage{hyperref}

\usepackage{comment}

\newtheorem{assumption}{Assumption}
\newtheorem{proposition}{Proposition}
\newtheorem{lemma}{Lemma}

\newtheorem{definition}{Definition}

\def\eE{\ensuremath{{\mathbb{E}}}}
\def\rR{\ensuremath{{\mathbb{R}}}}

\def\psim{\ensuremath{{p_{\textup{\tiny sim}}}}}
\def\qcl{\ensuremath{{q_{\textup{\tiny CL}}}}}
\def\qucl{\ensuremath{{q_{\textup{\tiny UCL}}}}}
\def\qscl{\ensuremath{{q_{\textup{\tiny SCL}}}}}
\def\qhucl{\ensuremath{{q_{\textup{\tiny H-UCL}}}}}
\def\qhscl{\ensuremath{{q_{\textup{\tiny H-SCL}}}}}
\def\qhcol{\ensuremath{{q_{\textup{\tiny Hcol}}}}}
\def\Hcalucl{\ensuremath{{\mathcal{H}_{\textup{\tiny UCL}}}}}
\def\Hcalscl{\ensuremath{{\mathcal{H}_{\textup{\tiny SCL}}}}}
\def\Hcalhucl{\ensuremath{{\mathcal{H}_{\textup{\tiny H-UCL}}}}}
\def\Hcalhscl{\ensuremath{{\mathcal{H}_{\textup{\tiny H-SCL}}}}}
\def\Hcalhcol{\ensuremath{{\mathcal{H}_{\textup{\tiny Hcol}}}}}

\def\alphascl{\ensuremath{{\alpha_{\textup{\tiny SCL}}}}}
\def\alphahucl{\ensuremath{{\alpha_{\textup{\tiny H-UCL}}}}}
\def\alphahscl{\ensuremath{{\alpha_{\textup{\tiny H-SCL}}}}}
\def\alphahcol{\ensuremath{{\alpha_{\textup{\tiny Hcol}}}}}

\allowdisplaybreaks

\usepackage{xcolor}
\usepackage{bbm}
\newif\ifcomments
\commentstrue
\ifcomments
    \usepackage{ulem}
    \def\picomment#1{{$\!$\color{magenta} [PI: #1]}}

    \def\sacomment#1{{$\!$\color{blue} [SA: #1]}}

    \def\rjcomment#1{{$\!$\color{brown} [RJ: #1]}}

    \def\tncomment#1{{$\!$\color{brown} [TN: #1]}}

\else 
    \def\picomment#1{}

    \def\sacomment#1{}

    \def\rjcomment#1{}

    \def\tncomment#1{}
    
\fi
\usepackage{newfloat}
\usepackage{listings}
\newfloat{listing}{tb}{lst}{}
\floatname{listing}{Listing}

\def\BibTeX{{\rm B\kern-.05em{\sc i\kern-.025em b}\kern-.08em
    T\kern-.1667em\lower.7ex\hbox{E}\kern-.125emX}}

\begin{document}

\title{Supervised Contrastive Learning \\ with Hard Negative Samples}

\author{\IEEEauthorblockN{1\textsuperscript{st} Ruijie Jiang}
\IEEEauthorblockA{\textit{Dept. of ECE} \\
\textit{Tufts University}\\
Medford, USA \\
Ruijie.Jiang@tufts.edu}
\and
\IEEEauthorblockN{1\textsuperscript{st} Thuan Nguyen}
\IEEEauthorblockA{\textit{Dept. of CS} \\
\textit{Tufts University}\\
Medford, USA \\
Nguyen.Thuan@tufts.edu}
\and
\IEEEauthorblockN{2\textsuperscript{nd} Prakash Ishwar}
\IEEEauthorblockA{\textit{Dept. of ECE} \\
\textit{Boston University}\\
Boston, USA \\
pi@bu.edu}
\and
\IEEEauthorblockN{3\textsuperscript{rd} Shuchin Aeron}
\IEEEauthorblockA{\textit{Dept. of ECE} \\
\textit{Tufts University}\\
Medford, USA\\
shuchin@ece.tufts.edu}
}

\maketitle

\begin{abstract}
 Through minimization of an appropriate loss function such as the InfoNCE loss, contrastive learning (CL) learns a useful representation function by pulling positive samples close to each other while pushing negative samples far apart in the embedding space. The positive samples are typically created using ``label-preserving'' augmentations, i.e., domain-specific transformations of a given datum or anchor. In absence of class information, in unsupervised CL (UCL), the negative samples are typically chosen randomly and independently of the anchor from a preset negative sampling distribution over the entire dataset. This leads to class-collisions in UCL. Supervised CL (SCL), avoids this class collision by conditioning the negative sampling distribution to samples having labels different from that of the anchor. In hard-UCL (H-UCL), which has been shown to be an effective method to further enhance UCL, the negative sampling distribution is conditionally \emph{tilted}, by means of a \emph{hardening function}, towards samples that are closer to the anchor. Motivated by this, in this paper we propose hard-SCL (H-SCL) {wherein} the class conditional negative sampling distribution {is tilted} via a hardening function. Our simulation results confirm the utility of H-SCL over SCL with significant performance gains {in downstream classification tasks.} Analytically, we show that {in the} limit of infinite negative samples per anchor and a suitable assumption, the {H-SCL loss} is upper bounded by the {H-UCL loss}, thereby justifying the utility of H-UCL {for controlling} the H-SCL loss in the absence of label information. Through experiments on several datasets, we verify the assumption as well as the claimed inequality between H-UCL and H-SCL losses. We also provide a plausible scenario where H-SCL loss is lower bounded by UCL loss, indicating the limited utility of UCL in controlling the H-SCL loss. 
 \footnote{Our code is publicly available at \href{https://github.com/rjiang03/H-SCL}{https://github.com/rjiang03/H-SCL}.}

\end{abstract}


\section{Introduction}
Contrastive representation learning (CL) has received considerable attention in the machine learning literature as a method to learn  representations of data for use in downstream inference tasks, both in the absence of class information via unsupervised CL (UCL) \cite{van2018representation, arora2019theoretical}, as well as with known class labels via supervised CL (SCL) \cite{khosla2020supervised}. 
Contrastive learning has impacted a number of applications ranging from image classification \cite{tian2020contrastive,chen2020simple,huang2023accuracy}, text classification \cite{chen2022contrastnet,gouvea2022analyzing}, and natural language processing \cite{jiang2021interpretable,rethmeier2023primer}, to learning and inference with time-series data \cite{mohsenvand2020contrastive,nonnenmacher2022utilizing}. 
Contrastive learning  methods learn a representation map that pulls positive samples together while pushing the negative samples apart in the representation space 
by minimizing a suitable loss such as the widely-used InfoNCE loss {\cite{oord2018representation}}. 
Given an anchor datum, the positive samples are often constructed by applying domain-specific augmentations or transformations that are highly likely to preserve the latent label \cite{chen2020simple}. For example, crop, blur, rotation, and occlusion transformations for image data, and word masking for natural language processing (NLP) data. For a given augmentation mechanism, the performance of CL highly depends on the choice of the negative sampling mechanism that provides adequate \emph{contrast} with the given anchor. In UCL, the negative samples are typically chosen randomly and independently of the anchor from a preset negative sampling distribution over the entire dataset. This leads to class-collisions in UCL. Supervised CL (SCL), avoids this class collision via conditioning the negative sampling distribution on the label of the anchor. In hard-UCL (H-UCL), which has been empirically shown to be an effective method for further enhancing the effectiveness of UCL on downstream inference tasks \cite{robinson2021contrastive,tabassum2022hard,jiang2023hard}, the negative sampling distribution is conditionally \emph{tilted}, by means of a \emph{hardening function}, towards samples that are closer to the anchor {\cite{jiang2023neural}}. Motivated by the success of H-UCL, in this paper we propose hard-negative sampling for SCL, where we tilt the class conditional negative sampling distribution via a hardening function similarly to H-UCL.  To the best of our knowledge, this work is the first one that jointly combines the label information and the hard-negative sampling strategies to improve downstream performance. We make the following main contributions:
\begin{enumerate}
    \item Via extensive numerical comparisons on standard datasets we show that downstream performance of H-SCL is significantly higher compared to SCL. A preview of the results is shown in Fig.~\ref{fig: 1} and more results are provided in Sec.~\ref{sec: numerical results}.
    \item In {Sec.~\ref{sec:H-UCLvsH-SCL}}, for a general class of hardening functions recently introduced in {\cite{jiang2023neural}} for H-SCL and H-UCL, in the limit of negative samples going to infinity and under a suitable assumption, in Lemma~{\ref{lemma:hucl-scl-inequality}} we show that the H-SCL loss is upper bounded by the H-UCL loss. Since in our experiments H-SCL outperforms SCL we posit that this result takes a step towards theoretically justifying the utility of H-UCL over UCL, addressing a question left open in \cite{wu2020conditional}. 
    \item In Sec.~\ref{sec: numerical results}, we conduct experiments to numerically verify to what extent the assumption needed for the main theoretical result is satisfied. We also numerically verify Lemma \ref{lemma:hucl-scl-inequality}.
\end{enumerate}

\begin{figure}[t]
    \begin{center}
    \includegraphics[width=9 cm, height=6cm]{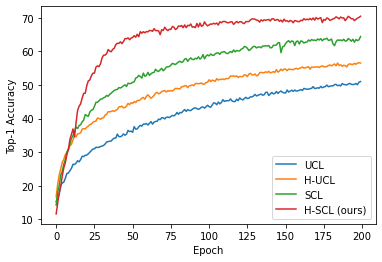}
 
    \caption{Top-1 accuracy (in \%) of UCL, H-UCL, SCL, and H-SCL on the CIFAR100 dataset.}
    \label{fig: 1}
    \end{center}
    
\end{figure}

The rest of the paper is organized as follows. In Sec.~\ref{sec: related work}, we discuss related work within the context of hard negative sampling for contrastive learning. Section~\ref{sec:problem_formulation} outlines the mathematical set-up for the UCL, SCL, H-UCL, and H-SCL scenarios. In Sec.~\ref{sec:H-UCLvsH-SCL} we state and prove the main theoretical result of this paper characterizing the relationship between the losses under the H-UCL and H-SCL settings along with the assumptions needed. Finally, we provide numerical results in Sec.~\ref{sec: numerical results} and conclude in Sec.~\ref{sec: conclusion}.

\section{Related work}
\label{sec: related work}

While design of positive sampling is also important in contrastive learning  \cite{tian2020makes}, in this work we focus mainly on negative sampling. For negative sample selection, recent works focus on designing ``hard" negative samples, \textit{i.e.,} negative samples coming from a different classes than the anchor, but close to the anchor. Robinson~\textit{et al.} \cite{robinson2021contrastive} derive a simple but practical hard-negative sampling strategy that improves the downstream task performance on image, graph, and text data. Tabassum~\textit{et al.} \cite{tabassum2022hard} introduce an algorithm called UnReMix which takes into account both the anchor similarity and the model uncertainty to select hard negative samples. Kalantidis \textit{et al.} \cite{kalantidis2020hard} propose a method called ``hard negative mixing" which synthesizes hard negative samples directly in the embedding space to improve the downstream-task performances. Although many studies observed that H-UCL outperforms UCL, there is no theoretical justification for this observation. Specifically, Wu \textit{et al.} \cite{wu2020conditional} observe that compared to the UCL loss, the H-UCL loss is, indeed, a looser lower bound of mutual information between two random variables derived from the dataset and raise the question ``why is a looser bound ever more useful" in practice? Our main theoretical result, Lemma~\ref{lemma:hucl-scl-inequality} takes a first step towards answering this question.
A recent work \cite{jiang2023neural} analyzes and establishes that the H-UCL loss is lower bounded by the UCL loss and the H-SCL loss\footnote{{Note that \cite{jiang2023neural} cites an earlier draft of this paper for H-SCL.}} is lower bounded by the SCL loss for general hard negative sampling strategies. In this paper we extend the results there and relate the H-UCL and H-SCL losses under a suitable technical assumption.

\begin{figure*}[h]
\centerline{\includegraphics[width=18cm, height=7.4cm]{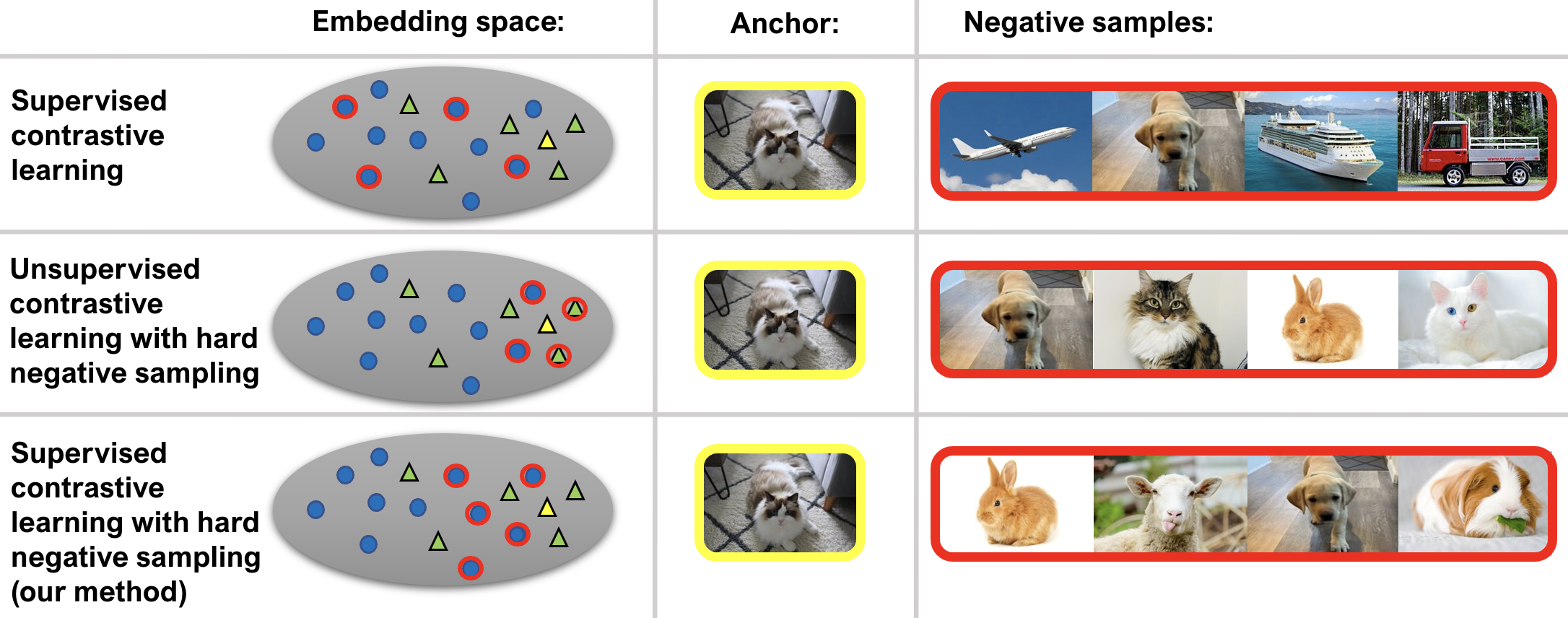}}
\caption{Schematic illustration of negative sampling strategies under H-UCL, SCL, and H-SCL settings in classifying species of cat. Top row (SCL): the negative samples (red rings) are randomly sampled from the set of circle samples which belongs to different classes of the anchor (yellow triangle). Middle row (H-UCL): the negative samples (red rings) are only selected from the neighbors of the anchor (yellow triangle). Since H-UCL prefers samples that are close to the anchor, it may select false negative samples (green triangles) which come from the same class as the anchor. Bottom row (H-SCL): the negative samples (red rings) are selected such that they are not only the ``true negative'' samples (circle samples) but also are close to the anchor (yellow triangle).}
\label{intro_frame}
\end{figure*}

\section{Problem Formulation}\label{sec:problem_formulation}

\noindent \textit{Notation and preliminaries:} We denote the \textit{sample space} (or \textit{input space}) by $\mathcal{X}$, the \textit{label space} by $\mathcal{Y}$, the unit-sphere in $\mathbb{R}^d$ by $S^{d-1}$, and the indicator function of an event $\mathcal{E}$ by $1(\mathcal{E})$. For integers $i, j$ with $i < j$, we define  $i:j := i, i+1, \ldots, j$ and $a_{i:j} := a_i, a_{i+1}, \ldots, a_j$. If $i > j$, $i:j$ and $a_{i:j}$ are ``null''. For any $f: \mathcal{X} \rightarrow S^{d-1}$ we define $g(x,x') := f^\top(x) f(x')/\gamma$, with $\gamma \in (0,\infty)$ a ``temperature parameter''. We measure the ``similarity'' of $u,v \in S^{d-1}$ by their inner product, with larger inner products corresponding to greater similarity. Since $||u-v||^2 = 2(1-u^\top v)$ for all $u,v \in S^{d-1}$, it follows that $u$ and $v$ are more similar if, and only if, they are more closer. For future reference, we note the following useful result:
\begin{proposition} \label{prop:tilting}
Let $p$ be a probability distribution over $\mathcal{Z}$ and $\rho: \mathcal{Z} \rightarrow [0,\infty)$ a nonnegative function such that $\alpha := \mathbb{E}_{z\sim p}[\rho(z)] \in (0,\infty)$. Then, 
\[
r(z) := \frac{\rho(z) p(z)}{\alpha}
\]
is also a probability distribution and for any measurable function $s: \mathcal{Z} \rightarrow \mathbb{R}$ we have
\[
\mathbb{E}_{z\sim r}[s(z)] = \frac{\mathbb{E}_{z\sim p}[\rho(z) s(z)]}{\alpha}
\]
\end{proposition}
\begin{proof}
Firstly, $r$ is nonnegative since $\rho$, $p$, and $\alpha$ are nonnegative. Since $\alpha$ is nonnegative, nonzero, and finite, we have
\begin{align*}
\int_{\mathcal{Z}} r(z) dz  =  \int_{\mathcal{Z}} \frac{\rho(z)}{\alpha} p(z) dz = \frac{\mathbb{E}_{z\sim p}[\rho(z)]}{\alpha} = 1.
\end{align*}
Therefore,
\begin{align*}
\mathbb{E}_{z\sim r}[s(z)] &= \int_{\mathcal{Z}} s(z) r(z) dz =  \int_{\mathcal{Z}} s(z) \frac{\rho(z)}{\alpha} p(z) dz \\  
&= \frac{\mathbb{E}_{z\sim p}[\rho(z) s(z)]}{\alpha}.
\end{align*}
\end{proof}

\subsection{Contrastive learning} \label{sec:contrastive_learning}

Contrastive learning assumes access to pairs of \textit{similar} samples $(x,x^+) \sim \psim(x,x^+)$, with $x$ referred to as the \textit{anchor} and $x^+$ as the \textit{positive sample}, and $k$ \textit{negative samples} 
$x^-_{1:k}$ that are conditionally independent and identically distributed (iid) given anchor $x$, with distribution $\qcl(x^-|x)$, and are presumably not similar (in a representation space) to the anchor $x$.

Let $y: \mathcal{X} \rightarrow \mathcal{Y}$ be a deterministic labeling function mapping the input space to a label space. For a sample $x\in\mathcal{X}$, $y(x)$ is the (groundtruth) label assigned to $x$. The labels are available during training in the supervised settings (SCL, H-SCL), but unknown in the unsupervised settings (UCL, H-UCL). 

The goal is to learn a representation function $f: \mathcal{X} \rightarrow S^{d-1}$, mapping the input space to the latent space of unit-norm vectors in $\mathbb{R}^d$, that minimizes a contrastive loss function 
\[
\mathcal{L}^{(k)}_{\textup{\tiny CL}}(f)  := \mathbb{E}_{(x,x^+)\sim \psim} \left[\mathbb{E}_{x^-_{1:k} \sim \,\textrm{iid}\,\qcl}
\left[\psi_k(x,x^+,x^-_{1:k},f)\right]\right]
\]
over some family of representation functions $\mathcal{F}$, e.g., all deep neural networks with a specified architecture. In this work we only consider the widely used InfoNCE contrastive loss function \cite{oord2018representation} given by
\[
\psi_k(x,x^+,x^-_{1:k},f) = \log\left(1+e^{-g(x,x^+)} \frac{1}{k}\sum_{j=1}^{k} e^{g(x,x^-_j)}\right).
\]

We will assume that for all $f \in \mathcal{F}$ and all $x \in \mathcal{X}$, we have 
$\mathbb{E}_{x^- \sim \qcl(x^-|x)} [e^{g(x,x^-)}] < \infty$. Then, in the limit as $k \rightarrow \infty$, by the strong law of large numbers,
\[
\psi_k(x,x^+,x^-_{1:k},f) \overunderset{a.s.}{k \rightarrow \infty}{\longrightarrow} \psi_\infty(x,x^+,f,\qcl)
\]
where
\[
\psi_\infty(x,x^+,f,\qcl) := \log\left(1+e^{-g(x,x^+)} \mathbb{E}_{x^- \sim \qcl} \left[e^{g(x,x^-)}\right]\right).
\]
Since all representation vectors have unit-norm, $|g(x,x^+)| = |f^\top(x)f(x^+)|/\gamma \leq 1/\gamma < \infty$. This implies that both $\psi_k(x,x^+,x^-_{1:k},f)$ and $\psi_\infty(x,x^+,f,\qcl)$ 
are globally bounded functions. From the dominated convergence theorem it follows that
\[
\mathcal{L}^{(k)}_{\textup{\tiny CL}}(f) \overunderset{k \rightarrow \infty}{}{\longrightarrow} \mathcal{L}^{(\infty)}_{\textup{\tiny CL}}(f)
\]
where
\[
\mathcal{L}^{(\infty)}_{\textup{\tiny CL}}(f) := \mathbb{E}_{(x,x^+)\sim \psim} \left[\psi_\infty(x,x^+,f,\qcl) \right] 
\]

\subsection{UCL and SCL settings} \label{sec:CL_settings}
Let $(x, x^+)$ be drawn from a joint distribution $\psim$. For a given $(v, v^+)$, the main difference between UCL, SCL, and H-UCL settings arises from the negative sampling distribution $\qcl(x^-|x)$ and the sampling strategy. 

\begin{enumerate}
    \item In the UCL setting, for all $x, x^+ \in \mathcal{X}$,
    \[
        \qcl(x^-|x) =  \qucl(x^-), 
    \]
for some probability distribution $\qucl(x^-)$ over $\mathcal{X}$. Thus in the UCL setting, the negative samples are selected independently of the anchor and positive samples. 
    \item In the SCL setting, for the given labeling function $y(\cdot)$ and all $x, x^- \in \mathcal{X}$,
    \[
         \qcl(x^-|x) =  \qscl(x^-|x) 
    \]
    where
    \begin{align*}
    \qscl(x^-|x) &:= \frac{1(y(x^-)\neq y(x))\,\qucl(x^-)}{\alphascl(x)}, \\
    \alphascl(x) &:= \eE_{x^-\sim \qucl}\left[1(y(x^-)\neq y(x))\right] 
    \end{align*}
    and we assume that $\alphascl > 0$ (this would be true if for all classes $y$, we have $\eE_{x^-\sim \qucl}\left[1(y(x^-)=  y\right] > 0)$.
    To enable comparison of the UCL and SCL settings, we assume that the $\qucl$ distribution used in the SCL setting is identical to the negative sampling distribution used in the UCL setting. Thus, the distribution of negative samples in the SCL setting is $\qucl$ conditioned on the event that the negative samples have labels \textit{different} from that of the anchor.  
\end{enumerate}

\subsection{H-UCL and H-SCL settings} \label{sec:H-CL_settings}

We consider the very general class of negative sample hardening mechanisms introduced in \cite{jiang2023neural} which is based on a \textbf{\textit{hardening function}} defined as follows.  
\begin{definition}[Hardening function]~\cite{jiang2023neural}
$\eta: \rR \rightarrow \rR$  is a hardening function if it is non-negative and nondecreasing throughout $\rR$.  
\end{definition}
\noindent Examples of hardening functions include the exponential tilting hardening function $\eta_{\rm \tiny exp}(t) := e^{\beta t}$, $\beta > 0,$ employed in \cite{robinson2021contrastive,jiang2023hard} and $\eta_{\rm \tiny thresh}(t) := 1(e^t \geq \tau)$ for some threshold $\tau$.
\begin{enumerate}
    \item In the H-UCL setting, for all $x,x^- \in \mathcal{X}$, all $f \in \mathcal{F}$, and a given hardening function $\eta(\cdot)$,
    \[
    \qcl(x^-|x) = \qhucl(x^-|x,f)
    \]
    where,
    \begin{align*}
    \qhucl(x^-|x,f) &:= \frac{\eta(g(x,x^-))\,\qucl(x^-)}{\alphahucl(x,f)}, \\
    \alphahucl(x,f) &:= \eE_{x^-\sim \qucl}\left[\eta(g(x,x^-))\right],
    \end{align*}
    and we assume that $\alphahucl(x,f) \in (0,\infty)$ for all $x\in \mathcal{X}$ and all $f \in \mathcal{F}$.
    The hardening function is nondecreasing. Thus negative samples $x^-$ that are more similar to the anchor $x$ in the representation space, i.e., $g(x,x^-)$ is large, are more likely to be sampled under $\qhucl$ than $\qucl$.
    
    \item In the H-SCL setting we utilize both hard-negative sampling and label information. This is the first key contribution of this paper. This is motivated by the effectiveness of hard-negative sampling strategies in H-UCL and the usefulness of label information in SCL. The main difference between H-SCL and other contrastive learning methods (UCL, H-UCL, and SCL) comes from the way the negative samples are selected. 
    
    Formally, in H-SCL, the positive pair $(x,x^+)$ is first sampled from $\psim$, i.e., using the same sampling strategy as in UCL, SCL, and H-UCL. Next, for the given labeling function $y(\cdot)$, all $x,x^+ \in \mathcal{X}$, all $f\in \mathcal{F}$, and a given hardening function $\eta(\cdot)$,
    \[
    \qcl(x^-|x) = \qhscl(x^-|x,f) 
    \]
    where,
    \begin{align*}
    &\qhscl(x^-|x,f) \\
    := &\frac{1(y(x^-)\neq y(x))\,\eta(g(x,x^-))\,\qucl(x^-)}{\alphahscl(x,f)} \\
    &\alphahscl(x,f) \\
    :=& \eE_{x^-\sim \qucl}\left[1(y(x^-)\neq y(x))\, \eta(g(x,x^-))\right],
    \end{align*}
    and we assume that $\alphahscl(x,f) \in (0,\infty)$ for all $x\in \mathcal{X}$ and all $f \in \mathcal{F}$.

    In other words, in the H-SCL setting, we only select a negative sample $x^-$ which simultaneously satisfies two conditions:
    \begin{itemize}
        \item[(i)] $x^-$'s label is different from $x$, i.e., $y(x^-) \neq y(x)$, and       
        \item[(ii)] $x^-$ is hard to discern from $x$.
    \end{itemize}  
\end{enumerate}

The top, middle, and bottom rows of Fig. \ref{intro_frame} illustrate the negative sampling strategies in SCL, H-UCL, and H-SCL, respectively. In Sec.~\ref{sec: numerical results}, we will numerically demonstrate the advantages of H-SCL compared to SCL, UCL, and H-UCL.

The contrastive losses in the UCL, SCL, H-UCL, and H-SCL settings are given by 
\begin{align}
\mathcal{L}^{(k)}_{\textup{\tiny UCL}}(f)  &:= \mathbb{E}_{(x,x^+)\sim \psim} \Big[ \mathbb{E}_{x^-_{1:k} \sim \,\textrm{iid}\,\qucl}
\Big[\psi_k(x,x^+,x^-_{1:k},f)\Big]\Big] \nonumber \\
\mathcal{L}^{(k)}_{\textup{\tiny SCL}}(f)  &:= \mathbb{E}_{(x,x^+)\sim \psim} \Big[ \mathbb{E}_{x^-_{1:k} \sim \,\textrm{iid}\,\qscl}
\Big[\psi_k(x,x^+,x^-_{1:k},f)\Big]\Big] \nonumber \\
\mathcal{L}^{(k)}_{\textup{\tiny H-UCL}}(f)  &:= \mathbb{E}_{(x,x^+)\sim \psim} \Big[ \mathbb{E}_{x^-_{1:k} \sim \,\textrm{iid}\,\qhucl}
\Big[\psi_k(x,x^+,x^-_{1:k},f)\Big]\Big] \nonumber \\
\mathcal{L}^{(k)}_{\textup{\tiny H-SCL}}(f)  &:= \mathbb{E}_{(x,x^+)\sim \psim} \Big[ \mathbb{E}_{x^-_{1:k} \sim \,\textrm{iid}\,\qhscl}
\Big[\psi_k(x,x^+,x^-_{1:k},f)\Big]\Big] \nonumber 
\end{align}
respectively and their limits as $k \rightarrow \infty$ by $\mathcal{L}^{(\infty)}_{\textup{\tiny UCL}}(f), \mathcal{L}^{(\infty)}_{\textup{\tiny SCL}}(f), \mathcal{L}^{(\infty)}_{\textup{\tiny H-UCL}}(f)$, and $\mathcal{L}^{(\infty)}_{\textup{\tiny H-SCL}}(f)$ respectively.

\section{Connection between H-SCL and H-UCL losses}
\label{sec:H-UCLvsH-SCL}
There is, in general, no known simple relationship between $\mathcal{L}_{\textup{\tiny H-SCL}}$ and $\mathcal{L}_{\textup{\tiny UCL}}$. In this section, however, we will show that under certain technical conditions $\mathcal{L}^{(\infty)}_{\textup{\tiny H-SCL}} \leq \mathcal{L}^{(\infty)}_{\textup{\tiny H-UCL}}$. This implies that when $k$ is large, minimizing $\mathcal{L}_{\textup{\tiny H-UCL}}$ can act as a proxy for minimizing $\mathcal{L}_{\textup{\tiny H-SCL}}$ whereas minimizing $\mathcal{L}_{\textup{\tiny UCL}}$ cannot in general. {As we will demonstrate in Sec.~\ref{sec: numerical results}, H-SCL empirically outperforms other contrastive learning methods.} Thus our theoretical results provide a plausible explanation for why H-UCL outperforms UCL in practice and partially answer an open question in \cite{wu2020conditional}.

For the given labeling function $y(\cdot)$, all $x,x^+ \in \mathcal{X}$, all $f\in \mathcal{F}$, and a given hardening function $\eta(\cdot)$, let
\begin{align*}
    \qhcol(x^-|x,f) &:= \frac{1(y(x^-)= y(x))\,\eta(g(x,x^-))\,\qucl(x^-)}{\alphahcol(x,f)} \\
    \alphahcol(x,f) &:= \eE_{x^-\sim \qucl}\left[1(y(x^-) =  y(x))\, \eta(g(x,x^-))\right],
\end{align*}
and assume that $\alphahcol(x,f) \in (0,\infty)$ for all $x\in \mathcal{X}$ and all $f \in \mathcal{F}$. Negative samples generated using $\qhcol$ have the same label as that of the anchor $y(x)$, i.e., we have a label collision, and are also hard to distinguish from the anchor $x$ in the representation space $\mathcal{F}$. 
%
\begin{proposition} \label{prop:Hucl-scl-col-decomp}
For all $x\in \mathcal{X}, f \in \mathcal{F}$ and any hardening function $\eta(\cdot)$ common to H-UCL, H-SCL, and Hcol  we have
\[
\alphahucl(x,f) = \alphahscl(x,f) + \alphahcol(x,f)
\]
\end{proposition}
\begin{proof} Adding
 \[
\alphahscl(x,f) = \eE_{x^-\sim \qucl}\left[1(y(x^-)\neq y(x))\, \eta(g(x,x^-))\right]
\]   
\[
\alphahcol(x,f) = \eE_{x^-\sim \qucl}\left[1(y(x^-) =  y(x))\, \eta(g(x,x^-))\right] \Rightarrow
\]
\[
\alphahscl + \alphahcol = \eE_{x^-\sim \qucl}\left[\eta(g(x,x^-))\right] = \alphahucl(x,f).
\]
\end{proof}
We make the following technical assumption:
\begin{assumption} \label{asp:key}  For any given $f \in \mathcal{F}$ and hardening function $\eta(\cdot)$, for all $x \in \mathcal{X}$,
\[
\eE_{x^- \sim \qhcol}\Big[e^{g(x,x^-)}\Big] \geq \eE_{x^- \sim \qhscl}\Big[e^{g(x,x^-)}\Big].
\]
\end{assumption}
Assumption~\ref{asp:key} asserts that in expectation, the exponentiated similarity (respectively, distance) between the anchor $x$ and hard-to-distinguish samples sharing the anchor's label is greater (respectively, smaller) than the exponentiated similarity (respectively, distance) between the anchor $x$ and hard-to-distinguish samples that also have a different label from the anchor.
%

In practice, Assumption~\ref{asp:key} is reasonable if the representation function $f$ is a ``good'' mapping, i.e., under the mapping $f$, samples having the same label are pulled closer to each other whereas samples having different labels are pushed far apart. {In Sec.~\ref{sec: numerical results}, we will provide some empirical evidence for Assumption~\ref{asp:key}.}

\begin{lemma}
\label{lemma:hucl-scl-inequality}
Under Assumption \ref{asp:key}, 
\begin{equation}
    \mathcal{L}^{(\infty)}_{\textup{ \tiny H-UCL}}\, \geq  \mathcal{L}^{(\infty)}_{\textup{ \tiny H-SCL}}. \nonumber
\end{equation}
\end{lemma}

\begin{proof} 
We will show that for any given $f \in \mathcal{F}$, hardening function $\eta(\cdot)$, and all $x \in \mathcal{X}$, 
\begin{eqnarray*}
\eE_{x^- \sim \qhucl} \big[ e^{g(x,x^-)} \big] 
\geq \eE_{x^- \sim \qhscl }  \big[ e^{g(x,x^-)}  \big]
\end{eqnarray*}
from which the desired inequality would follow since $\log(\cdot)$ is a strictly increasing function. Using Proposition~\ref{prop:tilting}, 
\begin{align}
&\eE_{x^- \sim \qhucl} \big[ e^{g(x,x^-)} \big]  \nonumber \\
&= \frac{\eE_{x^- \sim \qucl}[\eta(g(x,x^-))e^{g(x,x^-)}]}{\alphahucl} \nonumber \\
&= \frac{\eE_{x^- \sim \qucl}[1(y(x) = y(x^-)\eta(g(x,x^-))e^{g(x,x^-)}]}{\alphahucl} \nonumber \\
&{}\ \ \ \ + \frac{\eE_{x^- \sim \qucl}[1(y(x) \neq y(x^-)\eta(g(x,x^-))e^{g(x,x^-)}]}{\alphahucl} \nonumber \\
&= \frac{\alphahcol}{\alphahucl} \eE_{x^- \sim \qhcol} \big[ e^{g(x,x^-)} \big] 
+ \frac{\alphahscl}{\alphahucl} \eE_{x^- \sim \qhscl} \big[ e^{g(x,x^-)} \big] \nonumber \\
&\geq \frac{\alphahcol}{\alphahucl} \eE_{x^- \sim \qhscl} \big[ e^{g(x,x^-)} \big] 
+ \frac{\alphahscl}{\alphahucl} \eE_{x^- \sim \qhscl} \big[ e^{g(x,x^-)} \big]  \label{eq:bykeyassumption}\\
&= \left(\frac{\alphahcol + \alphahscl}{\alphahucl}\right) \eE_{x^- \sim \qhscl} \big[ e^{g(x,x^-)} \big] \nonumber \\
&= \eE_{x^- \sim \qhscl} \big[ e^{g(x,x^-)} \big] \label{eq:byalphadecomp}
\end{align}
where inequality (\ref{eq:bykeyassumption}) follows from Assumption~\ref{asp:key} and equality (\ref{eq:byalphadecomp}) follows from Proposition~\ref{prop:Hucl-scl-col-decomp}.
\end{proof}

Lemma~\ref{lemma:hucl-scl-inequality} shows that the loss function of H-UCL can be used as a proxy to optimize the loss function of H-SCL under certain conditions. Lemma~\ref{lemma:hucl-scl-inequality} requires $k \rightarrow + \infty$ which suggests using a large value of $k$ in practice. This is consistent with the numerical results in \cite{he2020momentum,wu2020conditional,zhuang2019local,laskin2020curl} where large values of $k$ lead to higher accuracies in downstream tasks.   

We shall now loosely explain why the loss function of UCL ($\mathcal{L}_{\textup{ \tiny UCL}}$) cannot upper bound the loss function of H-SCL ($\mathcal{L}_{\textup{ \tiny H-SCL}}$) for all hardening functions. Suppose we have a hardening function such that for any given $f \in \mathcal{F}$ and any $x \in \mathcal{X}$ we have
\[
\eE_{x^- \sim \qhscl} \big[ e^{g(x,x^-)} \big] 
\geq \eE_{x^- \sim \qucl }  \big[ e^{g(x,x^-)}  \big].
\]
Then, $\mathcal{L}_{\textup{ \tiny H-SCL}} \geq \mathcal{L}_{\textup{ \tiny UCL}}$ since $\log(\cdot)$ is a strictly increasing function. To design such a hardening function, let
\[
\tau(x,f) := \eE_{x^- \sim \qucl }  \big[ e^{g(x,x^-)}  \big]
\]
and define 
\[
\eta(g(x,x^-)) = 1(e^{g(x,x^-)} \geq \tau).
\]
Then, all the hard negative samples generated by $\qhscl$ will, by design, satisfy the inequality $e^{g(x,x^-)} \geq \tau$. If there is a nonzero probability of at least some of them having a label different from that of the anchor, then we would have  $\eE_{x^- \sim \qhscl} \big[ e^{g(x,x^-)} \big]  \geq \tau(x,f) = \eE_{x^- \sim \qucl }  \big[ e^{g(x,x^-)}  \big]$.

\subsection{Thresholded similarity hardening function} 
\label{sec:threshold_harden} 

To gain better intuition for the theoretical concepts and results we have developed, we now consider a special hardening function based on thresholding the similarity between samples, specifically,
\[
\eta_{\rm \tiny thresh}(t) := 1(e^t \geq \tau)
\]
where $\tau > 0$ is a threshold that controls the hardness of negative samples. A large value of $\tau$ makes it harder to distinguish between the anchor and negative samples in the representation space.
For this hardening function, 
\begin{align*}
\alphascl &= \qucl(\Hcalscl) \\
\alphahucl &= \qucl(\Hcalhucl) \\
\alphahscl &=  \qucl(\Hcalhscl) \\
\alphahcol &= \qucl(\Hcalhcol)
\end{align*}
where for any set $\mathcal{H}$,
\[
\qucl(\mathcal{H}) := \text{Pr}_{x^-\sim\qucl}(x^- \in \mathcal{H}),
\]
and
\begin{align*}
\Hcalscl(x) &:= \big\{x^- \in \mathcal{X} \,|\, y(x^-) \neq y(x) \big\}. \\
\Hcalhucl(x,f,\tau) &:= \big\{x^- \in \mathcal{X} | e^{g(x,x^-)} \geq \tau  \big\} \\
\Hcalhscl(x,f,\tau) &:= \Hcalscl(x) \cap \Hcalhucl(x,f,\tau) \\
\Hcalhcol(x,f,\tau) &:= \Hcalscl^{\!\!\!\!\!\!\!c\ \;}(x) \cap \Hcalhucl(x,f,\tau)
\end{align*}
and ``$^c$'' denotes set complement. The negative sampling distributions for SCL, H-UCL, H-SCL, and Hcol are given by
\begin{align*}
    \qscl(x^-|x) = 
    \begin{cases} 
        \frac{\qucl(x^-)}{\qucl(\Hcalscl)} &x^- \in \Hcalscl(v) \\
        0 & \text{otherwise}
    \end{cases}
\end{align*}
\begin{align*}
    \qhucl(x^-|x,f,\tau) = 
    \begin{cases} 
        \frac{\qucl(x^-)}{\qucl(\Hcalhucl)} &x^- \in \Hcalhucl \\
        0 & \text{otherwise}
    \end{cases}
\end{align*}
\begin{align*}
    \qhscl(x^-|x,f,\tau) = 
    \begin{cases} 
        \frac{\qucl(x^-)}{\qucl(\Hcalhscl)} & \text{if } x^- \in \Hcalhscl \\
        0 & \text{otherwise}
    \end{cases}
\end{align*}
\begin{align*}
    \qhcol(x^-|x,f,\tau) = 
    \begin{cases} 
        \frac{\qucl(x^-)}{\qucl(\Hcalhcol)} & x^- \in \Hcalhcol \\
        0 & \text{otherwise}
    \end{cases}
\end{align*}
respectively.
We note that $\Hcalscl(x)$ is the set of all samples having labels different from that of the anchor $x$. 
The set $\Hcalhucl(x,f,\tau)$ consists of all samples for which $e^{g(x,x^-)} \geq \tau$, i.e., samples whose similarity with the anchor $x$ in the representation space is greater than or equal to $\gamma \tau$, {where $\gamma$ is the temperature parameter of $g$,} 
and are therefore harder to distinguish from the anchor than other samples.
The set $\Hcalhscl(x,f,\tau)$ consists of all samples having labels different from that of the anchor $x$ and for which $e^{g(x,x^-)} \geq \tau$, i.e., they are also hard to distinguish from the anchor in the representation space.
Finally, the set $\Hcalhcol(x,f,\tau)$ consists of all samples having the \textit{same} label as the anchor's, i.e.,  $y(x)$, but are hard to distinguish from the anchor in the representation space, specifically, $e^{g(x,x^-)} \geq \tau$. 

Figure~\ref{fig: 2} illustrates the relationships between $\Hcalucl, \Hcalscl, \Hcalhucl, \Hcalhscl, \Hcalhcol$.

\begin{figure}[t]
    \begin{center}
    \includegraphics[width=8.5 cm, height=5.2cm]{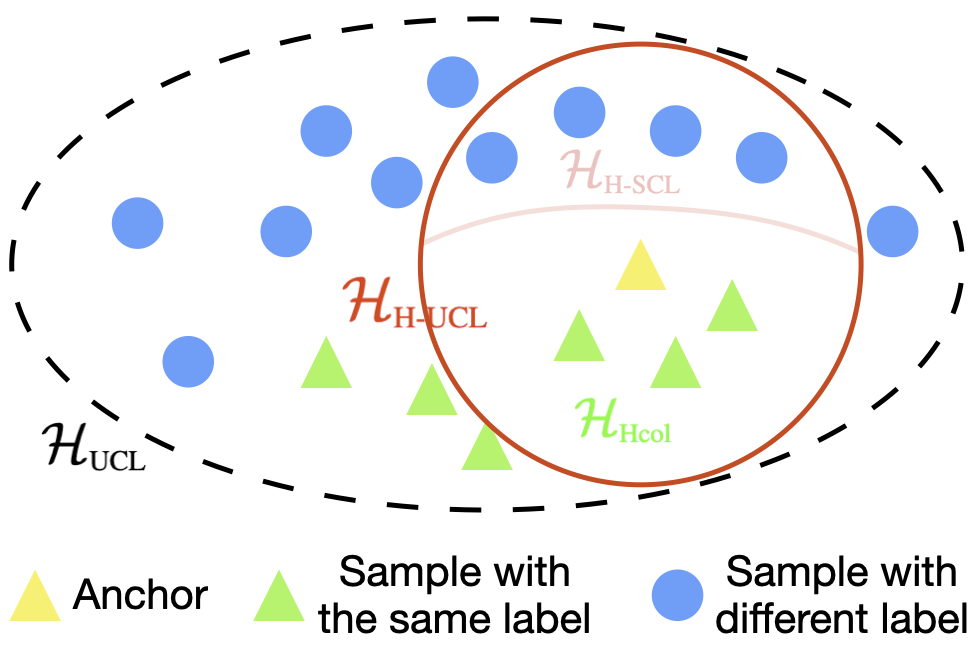}
    \caption{ 
    For a given representation function $f$, anchor $x$ (yellow triangle) 
    and a threshold $\tau$, $\Hcalhucl(x,f,\tau)$ contains all the samples $x^-$ that satisfy the constraint $e^{g(x,x^-)} \geq \tau$ (samples within the solid-line circle in the figure) that  which are difficult to distinguish from the anchor in the representation space. $\Hcalhscl(x,f,\tau)$ is a subset of $\Hcalhucl(x,f,\tau)$ and only contains samples that are hard to distinguish from the anchor and have labels different from the anchor's (blue discs within the solid-line circle). $\Hcalhcol(x,f,\tau)$ only contains samples that are hard to distinguish from the anchor and have the same label as the anchor (triangles within the solid-line circle). The set $\Hcalscl(x)$ consists of all samples having labels different from the anchor's, irrespective of whether they are easy or hard to distinguish from the anchor (all blue discs). 
    } 
    \label{fig: 2}
    \end{center}
\vglue -4ex    
\end{figure}

\section{Numerical results}
\label{sec: numerical results}

In this section, we {first} demonstrate the efficiency of H-SCL over other competing methods on four image datasets. In addition, we also empirically verify Assumption~\ref{asp:key} which supports our claim that  $\mathcal{L}^{(k)}_{\textup{\tiny H-SCL}} \leq \mathcal{L}^{(k)}_{\textup{\tiny H-UCL}}$. {Since we are using stochastic gradient descent methods, in all the experiments for each batch, $k$ is chosen to be the number of all negative samples in a given batch, which may vary across batches and with batch sizes for different datasets.} Then we present additional experimental results for 5 graph datasets.

%

\begin{figure*}[ht]
    \centering
    \begin{minipage}[b]{0.24\textwidth}
        \includegraphics[width=\textwidth]{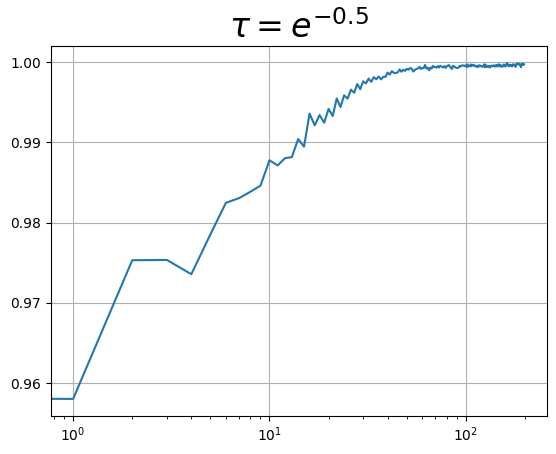}
    \end{minipage}
    \hfill
    \begin{minipage}[b]{0.24\textwidth}
        \includegraphics[width=\textwidth]{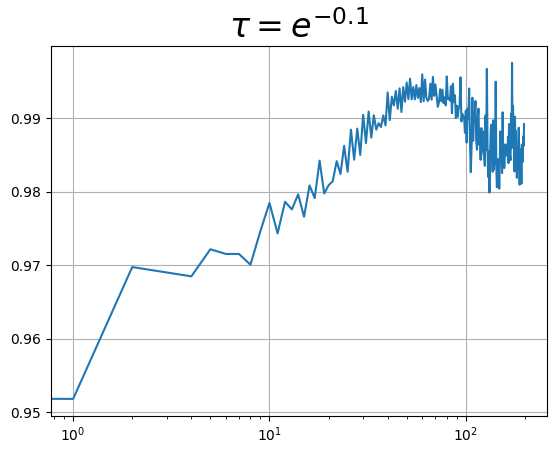}
    \end{minipage}
    \hfill
    \begin{minipage}[b]{0.245\textwidth}
        \includegraphics[ width=\textwidth]{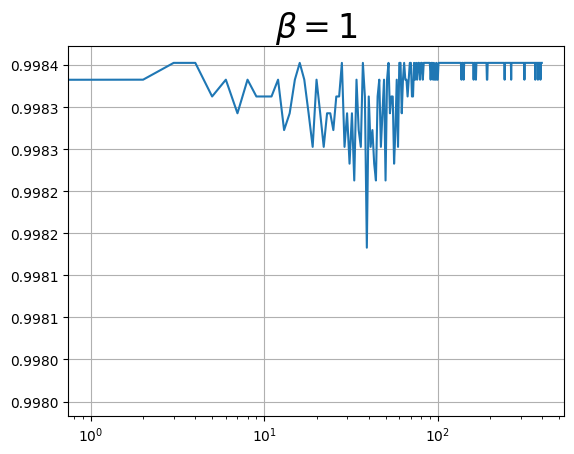}
    \end{minipage}
    \hfill
    \begin{minipage}[b]{0.245\textwidth}
        \includegraphics[width=\textwidth]{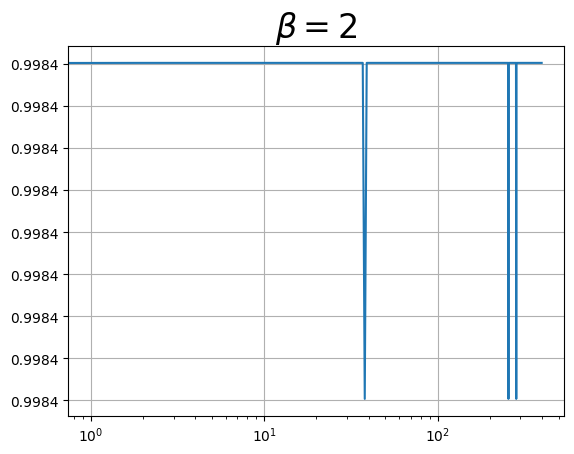}
    \end{minipage}
     \caption{{Fraction of anchors satisfying Assumption~\ref{asp:key} at the end of each epoch} for $\tau = e^{-0.5}$ (first figure), $\tau = e^{-0.1}$ (second figure), $\beta = 1$ (third figure) and $\beta = 2$ (forth figure).}\label{fig:asp}
\end{figure*}

\begin{figure*}[ht]
    \centering
    \begin{minipage}[b]{0.245\textwidth}
        \includegraphics[width=\textwidth]{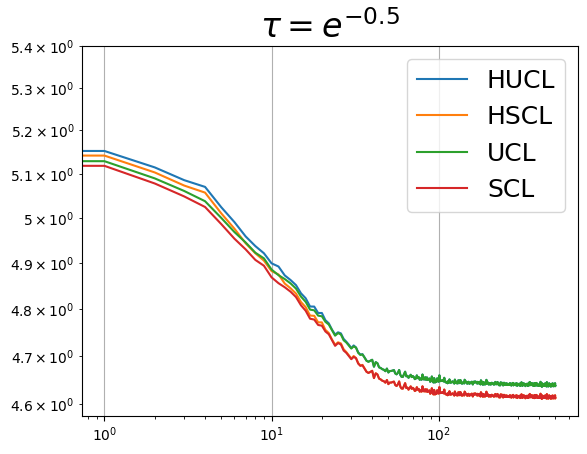}
    \end{minipage}
    \hfill
    \begin{minipage}[b]{0.245\textwidth}
        \includegraphics[width=\textwidth]{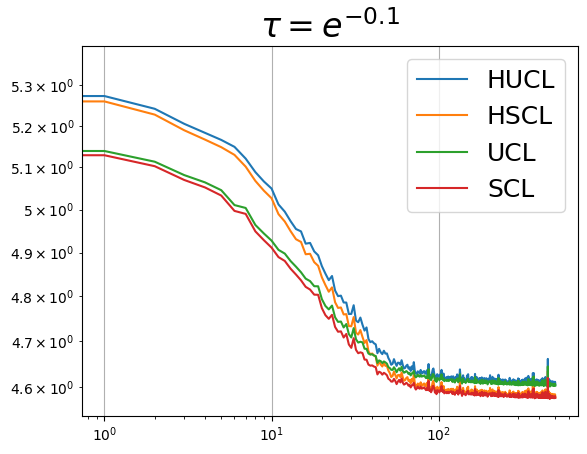}
    \end{minipage}
    \hfill
    \begin{minipage}[b]{0.24\textwidth}
        \includegraphics[ width=\textwidth]{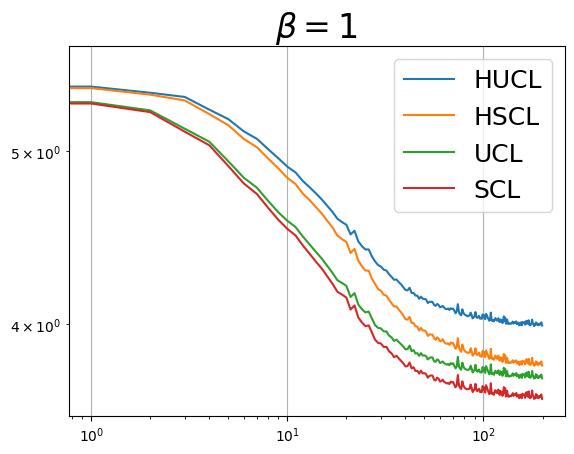}
    \end{minipage}
    \hfill
    \begin{minipage}[b]{0.24\textwidth}
        \includegraphics[width=\textwidth]{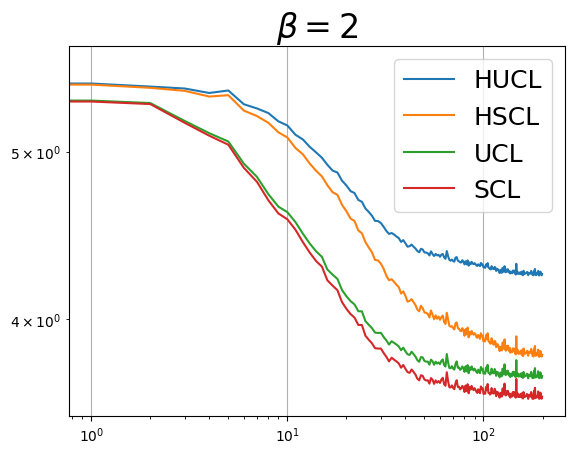}
    \end{minipage}
     \caption{Comparison of four different loss functions across epochs, with $\tau = e^{-0.5}$ (first figure), $\tau = e^{-0.1}$ (second figure), $\beta = 1$ (third figure) and $\beta = 2$ (forth figure)}\label{fig:loss}
\end{figure*}

\subsection{Image datasets}
We evaluated UCL, H-UCL, SCL, and H-SCL on the STL10 
\cite{coates2011analysis}, CIFAR10, and CIFAR100  \cite{krizhevsky2009learning} datasets for visual representation learning which contain images with 10, 10, and 100 classes, respectively. 

\textbf{Experiment setup:}
We adopt the simulation set-up and practical implementation from \cite{khosla2020supervised}. For UCL and H-UCL, the positive samples are generated using augmentations (crop, flip, color-jitter, and Gaussian noise) only, while for the SCL and H-SCL the positive sample are generated using both the augmentations and the label information. We employ two different methods for selecting hard negative samples via hardening functions. In the first method we use $\eta_{\rm \tiny thresh}(t) = 1(e^{t} \geq \tau)$ for sampling hard-negatives and call it the H-SCL$(\tau)$ method and in the second method we use $\eta_{\rm \tiny exp}(t) = e^{\beta t}$ for sampling hard-negatives and call it the  H-SCL$(\beta)$ method. For both methods we use the large negative sample limit of the Info-NCE loss function: 
\begin{align}
\label{eq:NCE_loss_exp}
\log\left(1+ M \,\, e^{-g(x,x^+)} \mathbb{E}_{x^- \sim \qcl} \left[e^{g(x,x^-)}\right]\right),
\end{align}
where we approximate the inner expectation by averaging over all the negative samples in a given batch.
%
%
The additional parameter $M$ is a positive scalar that is used in benchmark implementations in \cite{chen2020simple,robinson2021contrastive,khosla2020supervised} and for a fair comparison we include it in our simulations.

We use the simCLR set-up \cite{chen2020simple} with the projection head dimension of $128$ with ResNet-50 \cite{He2015} architecture to parameterize the representation function. 
%
After fixing the representation function generated by the trained ResNet-50, we train a linear classifier using {the available labeled data for each dataset} and report the classification accuracies.

\begin{table}[ht]
\centering
\caption{The best accuracies of tested methods on four datasets (in \%).}
\label{tbl:image}
\begin{tabular}{|c|c|c|c|c|}
\hline
\textbf{Method} & \textbf{STL10} & \textbf{CIFAR10} & \textbf{CIFAR100} & \textbf{TinyImageNet} \\
\hline
{UCL~\cite{chen2020simple}} & $64.36 \pm 0.92$ & 89.16 & 64.02 & 53.40 \\
{H-UCL~\cite{robinson2021contrastive}} & $67.82 \pm 1.41$ & 90.35 & 67.77 & 56.22 \\
{SCL} & $68.28 \pm 0.92$ & 93.46 & 71.68 & 62.06 \\
{H-SCL($\beta$)}& \textbf{72.52 $\pm$ 1.94} & \textbf{93.98} & \textbf{75.11} & \textbf{65.39} \\
{H-SCL($\tau$)}& 71.02 $\pm$ 2.03 & $92.95$ & $72.97$& $63.84$ \\
\hline
\end{tabular}
\end{table}

\textbf{Training procedure:} 
All models are trained for 200 epochs with a batch size of 512. We use the Adam optimizer with a learning rate of 0.001 and weight decay of 10$^{-6}$. We set $\gamma$ to $0.5$ following~\cite{robinson2021contrastive} for a fair comparison. \textcolor{black}{For the H-SCL($\tau$) method, we set $\tau = e^{l(\text{start, end, epoch})/\gamma}$ , where $l(\text{start, end, epoch})$ is a function defined as}
\begin{align}
    l(\text{start, end, epoch}) = \text{start} + \frac{\text{epoch}-1}{199}(\text{end}-\text{start})
\end{align}
we searched start from a set of$\{-0.5, -0.3, -0.1\}$, the end from a set $\{-0.1, 0, 0.1\}$ testing all nine possible combinations. 
%
%
For the H-SCL($\beta$) method, there is only one hyper-parameter that needs to be tuned namely $\beta$. Here, we perform a grid search of $\beta$ over the set $\{ 0.1, 0.5, 1, 2, 5\}$ and set $M$ (in Eq. \eqref{eq:NCE_loss_exp})  equal to the batch size minus $2$ {following the standard implementations \cite{chen2020simple,robinson2021contrastive,khosla2020supervised}}.
We used NVIDIA A100 32 GB GPU for computations and it takes about 10 hours to train one model (200 epochs) for each dataset. 
Since labeled STL10 is a small dataset, we repeated our experiment five times on this dataset and reported the average accuracy together with its standard deviation.

\textbf{Results: } 
\label{sec: 6-4}
Table~\ref{tbl:image} compares the best (after grid search for the best parameters) accuracies of UCL, H-UCL, SCL, H-SCL($\beta$) and H-SCL($\tau$) attained on four image datasets. As seen, both H-SCL($\beta$) and H-SCL($\tau$) methods are better than the baseline for most cases. H-SCL($\beta$) consistently outperforms other methods with margins of at least 3\% points on CIFAR100 and 4\% points. However, H-SCL($\beta$) is just slightly better than SCL on CIFAR10. The accuracies of the tested methods on CIFAR100 as a function of epochs is shown in Fig.~\ref{fig: 1}. We can observe that H-SCL($\beta$) only requires less than 50 epochs to achieve the same accuracy as SCL at 200 epochs.

\subsection{Verification of Assumption \ref{asp:key} and Lemma \ref{lemma:hucl-scl-inequality}}
Note that Lemma \ref{lemma:hucl-scl-inequality} requires that the positive sampling distribution be the same for H-SCL and H-UCL. To ensure this we use both augmentation and label information to generate the positive samples for H-UCL and H-SCL.
In order to verify Assumption \ref{asp:key}, we compute the fraction of anchors that satisfy the Assumption~\ref{asp:key} at the end of each epoch on the CIFAR100 dataset using both the H-SCL$(\tau)$ and H-SCL$(\beta)$ methods and plot the fraction against epochs in Fig.~\ref{fig:asp}. Our results indicate that, for both methods, this assumption is satisfied for over 95\% of all anchors across all epochs.

Finally, Fig.~\ref{fig:loss} empirically confirm the correctness of Lemma~\ref{lemma:hucl-scl-inequality}. For both H-SCL$(\tau)$ and H-SCL$(\beta)$ methods, $\mathcal{L}_{\textup{\tiny H-SCL}}$ (orange curves) is always upper bounded by $\mathcal{L}_{\textup{\tiny H-UCL}}$ (blue curves). However, the relationship between $\mathcal{L}{\textup{\tiny H-SCL}}$ and $\mathcal{L}{\textup{\tiny UCL}}$ is not consistent. {Specifically, as shown in the first two subplots in Fig.~\ref{fig:loss}, UCL loss (green curve) is less than H-SCL loss (orange curve) in the earlier epochs and is greater than the H-SCL loss in the later epochs.}


\subsection{Graph dataset}

We also applied our method to learn graph representations on five graph datasets: MUTAG, ENZYMES, PTC, IMDB-BINARY, and IMDB-MULTI by \cite{morris2020tudataset}. We employ InfoGraph \cite{sun2019infograph} as a baseline {UCL} method.


\textbf{Training Procedure:} As the H-SCL$(\beta)$ method is consistently better than the H-SCL$(\tau)$ method, for the graph dataset, we only conduct the simulation using the H-SCL$(\beta)$ method. 
%
We search for the best values of $\beta$ over the set $\{1,2,10\}$, which is also used in \cite{robinson2021contrastive}.
%
%
We report the the accuracy for the best value of $\beta$ for each dataset. All models are trained for 200 epochs and we use the Adam optimizer with a learning rate $0.01$. We used the 3-layer GIN \cite{xu2018how} for the representation function with a representation dimension equal to $32$.
Then we train an SVM classifier based on the learned graph-embedding. Each model is trained 10 times with 10-fold cross-validation.

\textbf{Result: } We report the performance accuracy of the different methods in Table~\ref{tbl:graph} with boldface numbers indicating the best performance for each dataset. We observe that H-SCL($\beta$) is consistently better than other methods across 5 datasets.


\begin{table}[!htp] 
\caption{\label{tbl:graph} Accuracy on graph datasets.}
\begin{center}
\begin{tabular}{|c|ccccc|} 
\hline
 Method &  MUTAG & ENZYMES & PTC & IMDB-B & IMDB-M \\ \hline
 UCL~\cite{sun2019infograph}  & 86.8 & 50.4 & 55.3 & 72.2 & 49.6 \\ \hline
 H-UCL~\cite{robinson2021contrastive}& \textbf{87.2}& 50.4 & 57.3& 72.8 & 49.6 \\ \hline 
 SCL  & 86.9& 50.4& 55.8& 72.4& {49.9}\\ \hline
 H-SCL ($\beta$)& \textbf{87.2}& \textbf{50.7}& \textbf{57.7}& \textbf{73.0}& \textbf{50.1}\\ \hline
\end{tabular}
\end{center}
\end{table}

\section{Conclusion and discussion}
\label{sec: conclusion} 
In this paper we introduced hard-negative supervised contrastive learning (H-SCL) which utilizes both label information and hard-negative sampling to improve downstream performance. On several real-world datasets we empirically demonstrated that H-SCL can substantially improve performance of downstream tasks compared to other contrastive learning approaches. We showed that in the asymptotic setting where the number of negative samples goes to infinity and a technical assumption, the hard unsupervised contrastive learning loss upper bounds the hard supervised contrastive learning loss. Our future work aims to weaken the technical assumption that is required for this relationship to hold true. We further aim to establish similar results in the non-asymptotic setting having a finite number of negative samples.

\section{Acknowledgements}
This research is supported by NSF:DRL:1931978.

\bibliographystyle{IEEEtran}
\bibliography{IJCNN}

\end{document}